\documentclass[sigconf]{aamas}

\usepackage{balance} 
\usepackage{latexsym}
\usepackage{amsmath}
\usepackage{booktabs}
\usepackage{enumitem}
\usepackage{graphicx}

\setcopyright{ifaamas}
\acmConference[AAMAS '25]{Proc.\@ of the 24th International Conference
on Autonomous Agents and Multiagent Systems (AAMAS 2025)}{May 19 -- 23, 2025}
{Detroit, Michigan, USA}{Y.~Vorobeychik, S.~Das, A.~Nowe (eds.)}
\copyrightyear{2025}
\acmYear{2025}
\acmDOI{}
\acmPrice{}
\acmISBN{}

\acmSubmissionID{169}

\title[]{Causes and Strategies in Multiagent Systems}

\author{Sylvia S. Kerkhove}
\affiliation{
  \institution{Utrecht University}
  \city{Utrecht}
  \country{The Netherlands}}
\email{s.s.kerkhove@uu.nl}

\author{Natasha Alechina}
\affiliation{
  \institution{Open University}
  \city{Heerlen}
  \country{The Netherlands}  
  }
  \affiliation{
  \institution{Utrecht University}
  \city{Utrecht}
  \country{The Netherlands}}
\email{natasha.alechina@ou.nl}

\author{Mehdi Dastani}
\affiliation{
  \institution{Utrecht University}
  \city{Utrecht}
  \country{The Netherlands}}
\email{m.m.dastani@uu.nl}

\begin{abstract}
Causality plays an important role in daily processes, human reasoning, and artificial intelligence. 
There has however not been much research on causality in multi-agent strategic settings. 
In this work, we introduce a systematic way to build a multi-agent system model, represented as a concurrent game structure, for a given structural causal model. In the obtained so-called causal concurrent game structure, transitions correspond to interventions on agent variables of the given causal model. The Halpern and Pearl framework of causality is used to determine the effects of a certain value for an agent variable on other variables. The causal concurrent game structure allows us to analyse and reason about causal effects of agents’ strategic decisions. We formally investigate the relation between causal concurrent game structures and the original structural causal models. 
\end{abstract}


\newcommand{\BibTeX}{\rm B\kern-.05em{\sc i\kern-.025em b}\kern-.08em\TeX}

\newcommand{\bs}{\backslash}
\newcommand{\vphi}{\varphi}

\newcommand{\set}[1]{\{ {#1} \}}
\newcommand{\csetting}{(\mathcal{M},\mathbf{u})}
\newcommand{\csettingint}[1]{(\mathcal{M}^{#1},\mathbf{u})}
\newcommand{\myvec}[1]{\mathbf{#1}}

\begin{document}


\pagestyle{fancy}
\fancyhead{}


\maketitle 


\section{Introduction}
\label{sec:introduction}

Causality plays an important role in Artificial Intelligence~\cite{pearl1995causal,halpern2016actual}. A specific type of causality, called `actual causality',  concerns causal relations between concrete events (e.g. throwing a specific rock shatters a specific bottle)~\cite{halpern2016actual}. There is still discussion on what the best definition of actual causality is (see \cite{halpern2016actual,Hall_2007,gladyshev2023dynamic} and \cite{Beckers_Vennekens_2018} for some of those definitions). 
However, most approaches like \cite{gladyshev2023dynamic} and \cite{Beckers_Vennekens_2018} use Pearl's \cite{pearl1995causal} structural model framework.
In this structural model framework, the world is modelled through  variables, which are divided in exogenous and endogenous variables. 
The former are variables whose values are determined by causes outside of the model and the latter are variables whose values are determined by the variables inside the model (both exogenous and endogenous variables).
The functional dependencies between variables are formalised through structural equations.
There also exists a rule-based approach that uses logical language to capture causal relations (see \cite{bochman2018actual} and \cite{lorni2023rule}), but we focus on the structural model framework due to its prominence in the literature \cite{Alechina_Halpern_Logan_2020,Beckers_Vennekens_2018,chockler2004responsibility,gladyshev2023dynamics}.

While causal models can in principle depict multi-agent systems by making a distinction between agent and environment events, they are less appropriate for reasoning about the abilities and strategies of agents.
Concurrent game structures (CGS) have been proposed to reason about agent interactions and strategies \cite{alur2002alternating}. These structures are graphs where nodes correspond to states of the world and edges, labelled with agents' actions, correspond to state transitions \cite{baier2008principles,gorrieri2017process}. 
In deterministic settings, an agent strategy specifies the actions to take by the agent.

Let us introduce an example of a causal model. 
Consider a semi-autonomous vehicle controlled jointly by a human driver and an automatic driving assistance system.
This driving assistance system is in turn supported by an obstacle detection system that signals to the driving assistant whether there is an obstacle in front of the vehicle.
Both the human driver and the driving assistant control the forward movement of the vehicle, though the human driver can always take full control.
In a scenario where there is an obstacle in front of the car, the obstacle causes the obstacle detection system to send a signal to the driving assistant. 
If the human driver is in a distracted state, this signal causes the driving assistant to avoid an accident.
This scenario can be described as a causal system, but can also be viewed as a multi-agent system where the obstacle detection system, the driving assistant and the human driver are all seen as agents that make decisions based on their state observations.

The fundamental relationship between structural causal models and multi-agent system models manifests itself in modelling phenomena such as responsibility for realising a certain outcome by a group of agents.
In the literature of multi-agent systems, both structural causal models and CGS are used to define the responsibility of a group of agents for an outcome \cite{chockler2004responsibility,yazdanpanah2019strategic}. Agents in a structural causal model are seen as responsible for an outcome if they have caused it \cite{chockler2004responsibility}. On the other hand, in a CGS a coalition of agents is deemed responsible for an outcome if they had a strategy to prevent it \cite{yazdanpanah2019strategic}. 
By establishing the relationship between structural causal models and CGS, different modeling approaches to multi-agent phenomena (e.g., responsibility) can be compared and unified. 

In this paper, we aim to establish a formal relationship between structural causal models and concurrent game structures by constructing a CGS for a given structural causal model such that if a group of agents is an actual cause for an outcome in the causal model, then this group had a strategy in the constructed CGS to prevent the outcome, provided the other agents act as prescribed by the causal model.
The CGS is built by distinguishing between agent and environment variables.
We consider the values of an agent variable as possible actions of the agent and interventions as agents' decisions. 
We provide several formal results on how strategies in this causal concurrent game structure (causal CGS) relate to the original structural causal model, establishing a formal relationship between structural causal models and CGS.
In particular, we show that a choice of actions by a group of agents is a cause of an outcome in a structural causal model (under the Halpern-Pearl definition of an actual cause) if and only if this set of agents has a strategy for the negation of the outcome in the corresponding causal CGS, provided the other agents act according to the causal model.
We believe that our framework will be beneficial for supporting causal inference in multi-agent systems, for example, for reasoning and attributing responsibility for certain outcomes to groups of agents. 

We will now first give some preliminaries on causality and concurrent game structures.
In Section \ref{sec:cgs from cm} we define the translation from a structural causal model to a causal CGS, after which, in Section \ref{sec:causality in CGS}, we show how causality in the structural causal model relates to agent strategies in the causal CGS.

\section{Background}
\label{sec:literature overview}
In this section, we introduce the structural causal model framework that we will use. We also shortly introduce concurrent game structures and give a formal definition of agent strategies. 

\begin{definition}[Causal Model, Causal Setting \cite{halpern2016actual}] \label{def:causal model}
    A \emph{causal model} $\mathcal{M}$ is a pair $(\mathcal{S},\mathcal{F})$, where $\mathcal{S}$ is a signature and $\mathcal{F}$ defines a set of structural equations, relating the values of the variables.
    A \emph{signature} $\mathcal{S}$ is a tuple $(\mathcal{U},\mathcal{V},\mathcal{R})$, where $\mathcal{U}$ is a set of exogenous variables, $\mathcal{V}$ is a set of endogenous variables and $\mathcal{R}$ associates with every variable $X\in \mathcal{U}\cup \mathcal{V}$ a nonempty set $\mathcal{R}(X)$ of possible values for $X$.
    
    A \emph{causal setting} is a tuple $(\mathcal{M},\mathbf{u})$, where $\mathcal{M}$ is a causal model and $\mathbf{u}$ a setting for the exogenous variables in $\mathcal{U}$.
\end{definition}

The \emph{exogenous variables} are variables whose values depend on factors outside of the model, their causes are not explained by the model \cite{halpern2016actual,Pearl2016causal}.
On the other hand, the \emph{endogenous variables} are fully determined by the variables in the model. 
Note that with $\myvec{u}$, we use the bold-face notation to denote that $\mathbf{u}$ is a tuple.
When we use this bold-face notation for capital letters $\myvec{X}$ and $\myvec{Y}$, we are slightly abusing notation by treating them both as tuples and as sets. This follows Halpern's use of the vector notation for both concepts \cite{halpern2016actual}.
This means that we can write $\myvec{X} = \myvec{x}$ to indicate that the first element of $\myvec{X}$ gets assigned the value of the first element of $\myvec{x}$ and so on, but that we can also write $\myvec{X}'\subseteq \myvec{X}$.

\begin{example}\label{ex:causal model}
    Consider the semi-autonomous vehicle example we discussed in the introduction. 
    We can model this example with exogenous binary variables $U_O$ that determines whether there will be an obstacle on the route, and $U_{Att}$ that determines whether the human driver is paying attention.
    For the endogenous variables we introduce the binary variables $O$, indicating that there is an obstacle, $Att$, indicating that the human driver is paying attention, $HD$ for whether the human driver keeps driving or brakes. Note that we use $HD$ when the human driver keeps driving ($\neg HD$ indicates that they brake). $ODS$, indicating that the obstacle detection system detects an obstacle, $DA$, for whether the driving assistant keeps driving or brakes. Note that we use $DA$ when the human driver keeps driving ($\neg DA$ indicates that they brake). And $Col$, indicating a collision.
    The set $\mathcal{U}$ is hence $\set{U_O,U_{Att}}$ and the set $\mathcal{V}$ is hence $\set{O,Att, HD, ODS,DA, Col}$.
    We consider all variables to be Boolean, so for any variable $X \in \mathcal{U} \cup \mathcal{V}$, $\mathcal{R}(X) = \set{0,1}$.

    The following structural equations are defined for this model:
    \begin{equation*}
        \begin{array}{ll}
            O := U_O & Att:= U_{Att}  \\
            HD := \neg O \vee (O \wedge \neg Att) & ODS:= O \\
            DA := HD \wedge \neg ODS & Col := DA \wedge HD \wedge O.
        \end{array}
    \end{equation*}
\end{example}

A \emph{causal network} is a directed graph with nodes corresponding to the causal variables in $\mathcal{V}$ (and $\mathcal{U}$) with an edge from the node labelled $X$ to the node labelled $Y$ if and only if the structural equation for $Y$ depends on $X$.
In other words, we put an edge from node $X$ to node $Y$ if and only if $X$ can influence the value of $Y$ \cite{halpern2005causes}.
We call $Y$ a \emph{descendant} of $X$ if the graph contains a path from $X$ to $Y$.

A model that has an acyclic causal network is called strongly recursive \cite{halpern2016actual}. In such models, a setting $\mathbf{u}$ of the exogenous variables $\mathcal{U}$ fully determines the values of all other (endogenous) variables. We call a causal model with an acyclic causal network recursive because the exogenous variables determine the values of the endogenous variables in a recursive manner. 
As Halpern explains, some endogenous variables only depend on exogenous variables, we call them first-level variables \cite{halpern2016actual}. 
They get their value directly from the causal setting.
After that there are the second-level variables, the endogenous variables that depend on both the first-level variables and possibly on the exogenous variables.
Likewise the third-level variables depend on the second-level variables, and possibly on the exogenous and the first-level variables, and so on for higher levels.
We only focus on strongly recursive models in this paper.

\begin{example}\label{ex:causal network}
    The causal network for the causal model as described in Example \ref{ex:causal model} is given in Figure \ref{fig:ex causal network} (the exogenous variables are not drawn).
    The graph makes it easy to see that the causal model is recursive, i.e. the causal network does not contain cycles.
    \begin{figure}[b]
        \centering
        \setlength{\unitlength}{1cm}
        \begin{picture}(5,3.05)(-0.5,0)
            \put(0,0){\circle*{0.1}}
            \put(0,1.5){\circle*{0.1}}
            \put(2,0){\circle*{0.1}}
            \put(2,1.5){\circle*{0.15}}
            \put(2,3){\circle*{0.1}}
            \put(4,1.5){\circle*{0.1}}

            \put(0,0){\vector(1,0){1.95}}
            \put(0,1.5){\vector(4,3){1.95}}
            \put(0,1.5){\vector(1,0){1.95}}
            \put(0,1.5){\vector(4,-3){1.95}}
            \put(2,0){\vector(4,3){1.95}}
            \put(2,0){\vector(0,1){1.45}}
            \put(2,3){\vector(4,-3){1.95}}
            \put(4,1.5){\vector(-1,0){1.95}}

            \put(-0.15,0){\makebox(0,0)[r]{{$Att$}}}
            \put(-0.15,1.5){\makebox(0,0)[r]{{$O$}}}
            \put(2.15,-0.05){\makebox(0,0)[l]{{$HD$}}}
            \put(2,1.65){\makebox(0,0)[b]{\Large{$Col$}}}
            \put(1.85,3.05){\makebox(0,0)[r]{{$ODS$}}}
            \put(4.15,1.5){\makebox(0,0)[l]{{$DA$}}}
        \end{picture}
        \caption{The causal network for the causal model for the semi-autonomous vehicle example described in Example \ref{ex:causal model}.}
        \label{fig:ex causal network}
        \Description{A graph depicting the causal network of the causal model corresponding to the running semi-automated vehicle example. The graph has 6 nodes, each corresponding to one of the variables of the causal model, Att, O, HD, ODS, DA, and Col. Att has an edge to HD, O has edges to HD, Col and ODS. HD has an edge to DA and Col, ODS has an edge to DA, and DA has an edge to Col. There are no other edges.}
    \end{figure}
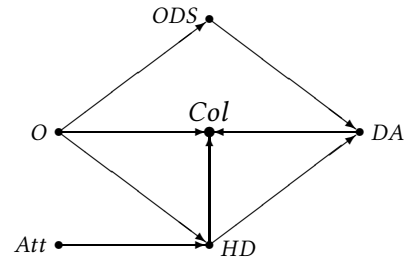
    We can also see the variable levels. 
    $O$ and $Att$ are first-level variables, they only depend on the exogenous variables.
    $HD$ and $ODS$ only depend on $O$ and $Att$ and hence are second-level variables. 
    $DA$ is a third-level variable, as it depends on second-level variables, and $Col$ is a fourth-level variable, as it depends on both $DA$ and lower-level variables.
\end{example}

Given a signature $\mathcal{S} = (\mathcal{U},\mathcal{V},\mathcal{R})$, a formula of the form $X=x$, for $X\in \mathcal{V}$ and $x\in\mathcal{R}(X)$ is called a \emph{primitive event} \cite{halpern2005causes,halpern2016actual}. These primitive events can be combined with the Boolean connectives $\wedge, \vee$ and $\neg$, to form a \emph{Boolean combinations of primitive events} \cite{halpern2005causes,halpern2016actual}. 
We follow Halpern and use $\csetting \models \phi$ to denote that formula $\phi$ holds given the values of all variables determined by the causal setting $\csetting$ (see~\cite{halpern2016actual} for details).
A \emph{causal formula} has the form $[Y_1\leftarrow y_1, ... , Y_k \leftarrow y_k]\varphi$, where $\varphi$ is a Boolean combination of primitive events, $Y_1,...,Y_k \in \mathcal{V}$ with
$Y_i = Y_j$ if and only if $i = j$, and $y_i\in\mathcal{R}(Y_i) \text{ for all } 1 \leq i \leq k$. Such a formula can be shortened to $[\myvec{Y}\leftarrow\mathbf{y}]\varphi$, and when $k=0$ it is written as just $\varphi$ \cite{halpern2005causes}.
$\csetting \models [\myvec{Y}\leftarrow\mathbf{y}](X = x)$ says that after an intervention that sets all variables of $\myvec{Y}$ to $\mathbf{y}$, it must be the case that $X = x$ holds in the causal setting $\csetting$ (see~\cite{halpern2005causes,halpern2016actual} for more details).
We call $\mathbf{y}$ a \emph{setting} for the variables in $\mathbf{Y}$.
We now have the necessary background to give the modified HP definition of causality:
\begin{definition}[modified HP Definition \cite{halpern2016actual}]\label{def:HP}
$\myvec{X} = \myvec{x}$ is an \emph{actual cause} of $\vphi$ in the causal setting $(\mathcal{M},\mathbf{u})$ if the following 3 conditions hold:
\begin{itemize}
    \item[\textnormal{AC1.}]  $(\mathcal{M},\mathbf{u})\models \myvec{X} = \myvec{x}$ and $(\mathcal{M},\mathbf{u})\models \vphi$;
    \item[\textnormal{AC2.}] There is a set $\myvec{W}$ of variables in $\mathcal{V}$ and a setting $\myvec{x}'$ of variables in $\myvec{X}$ s.t. if $(\mathcal{M},\mathbf{u})\models \myvec{W} = \mathbf{w}^\ast$, then $(\mathcal{M},\mathbf{u}) \models [\myvec{X} \leftarrow \myvec{x}', \myvec{W} \leftarrow \mathbf{w}^\ast] \neg \vphi$.
    \item[\textnormal{AC3.}] $X$ is minimal; there is no strict subset $\myvec{X}'$ of $\myvec{X}$ s.t. $\myvec{X}' = \myvec{x}'$ satisfies \textnormal{AC1} and \textnormal{AC2}, where $\myvec{x}'$ is the restriction of $\mathbf{x}$ to the variables in $\myvec{X}'$.
\end{itemize}
\end{definition}
If $\myvec{W} = \emptyset$, we call $\myvec{X} = \myvec{x}$ a \emph{but-for cause} of $\vphi$.

\begin{example}\label{ex:causes}
    Consider our semi-autonomous vehicle example again. 
    Take the causal setting where $\mathbf{u} = (1,0)$, i.e. $U_O = 1$, there is an obstacle on the route, and $U_{Att} = 0$, the human driver is not paying attention.
    Following the equations provided in Example~\ref{ex:causal model}, we have that $\csetting \models O \wedge \neg Att \wedge HD \wedge ODS \wedge \neg DA \wedge \neg Col$. 
    We want to know which agent was the cause of there being no collision.
    It turns out that both $ODS$ and $\neg DA$ are but-for causes of $\neg Col$, i.e., $\csetting \models [ODS \leftarrow 0] Col$ and $\csetting \models [DA \leftarrow 1] Col$.
    After all, if we intervene by turning off the object detection system $ODS$ (setting its value to $0$ in our model, i.e., replacing equation $ODS=1$ in our model with $ODS=0$, which is formally represented as $[ODS \leftarrow 0]$), the driving assistant $DA$ will no longer get a signal that there is an obstacle on the route.
    This gives $DA =1$, meaning that the driving assistant will not brake. 
    Because the human driver is distracted in this setting, they will also not brake, and so there will be a collision. 
    Similarly we can also directly intervene on the driving assistant by turning it off (setting its value to $1$, not braking, in our model by replacing the equation for $DA$ with $DA:=1$, represented by $[DA \leftarrow 1]$) and there will be a collision as well.
\end{example}

The aim of this work is to connect this concept of structural causal models and causality to concurrent game structures. We use the following definition of concurrent game structures:
\begin{definition}[Concurrent Game Structures \cite{alur2002alternating}]\label{def:concurrent game structures}
    A \emph{concurrent game structure} (CGS) is a tuple $GS = \langle N, Q, d, \delta, \Pi, \pi \rangle$ with the following components:
    \begin{itemize}
        \item A natural number $N \geq 1$ of agents. We identify the \emph{agents} with the numbers $1, . . . , N$.
        \item A finite set $Q$ of states. 
        \item For each agent $a \in \{1, . . . , N\}$ and each state $q \in Q$, a natural number $d_a (q) \geq 1$ of moves available at state $q$ to agent $a$. We identify the moves of agent $a$ at state $q$ with the numbers $1, . . . , d_a (q)$. For each state $q \in Q$, a move vector at $q$ is a tuple $\langle j_1, . . . , j_N\rangle$  such that $1 \leq j_a \leq d_a (q)$ for each agent $a$. Given a state $q \in Q$, we write $D(q)$ for the set $\{1, . . . , d_1 (q)\} \times \dots \times \{1, . . . , d_N (q)\}$ of \emph{move vectors}. The function $D$ is called \emph{move function}.  
        \item For each state $q \in Q$ and each move vector $\langle j_1, . . . , j_N\rangle \in D(q)$, a state $\delta(q, j_1, . . . , j_N ) \in Q$ that results from state $q$ if every agent $a \in \{1, . . . , N\}$ chooses move $j_a$ . The function $\delta$ is called \emph{transition function}.
        \item A finite set $\Pi$ of \emph{propositions}. 
        \item For each state $q \in Q$, a set $\pi (q) \subseteq \Pi$ of propositions true at q. The function $\pi$ is the \emph{labelling function}. 
    \end{itemize}
\end{definition}

When we have a CGS, we can reason about what the optimal actions for a coalition of agents would be in a certain situation. We often use the concept of strategies for this.

\begin{definition}[Strategy in Concurrent Game Structures \cite{alur2002alternating}]
    Given a concurrent game structure $S = \langle N, Q, d, \delta, \Pi, \pi \rangle$, a \emph{strategy} for agent $a\in\set{1,...,N}$ is a function $f_a$, that maps any (non-empty) finite sequence $\lambda$ of states in $Q$ to an action the agent can take at the last state of the sequence. I.e. if $q$ is the last state of $\lambda$, then $f_a(\lambda) \leq d_a(q)$.
    We write $F_A = \{f_a\ | \ a \in A\}$ for a set of strategies of the agents in $A\subseteq \set{1,...,N}$. 
\end{definition}

We now have all preliminaries ready to move on and combine causality with concurrent game structures.

\section{From Causal Model to CGS}
\label{sec:cgs from cm}
The goal of this paper is to define a systematic approach to generate a causal CGS based on a strongly recursive structural causal model. The motivation is that we want to compare the strategic ability of coalitions of agents to realise outcomes to causes in the causal model.
Similar translations have been attempted by \cite{gladyshev2023dynamics,baier2021game-theoretic} and \cite{hammond2023ReasoningCausalityGames}.

Gladyshev et al. make, like us, a distinction between agent and environment variables, and they also construct a CGS that takes the causal structure between agents' decision and environment variables into account \cite{gladyshev2023dynamics} . 
However, they take a `zoomed out' approach to the causal model by considering every state in the CGS as a causal model. 
In contrast, in this paper, we are interested in the specific variable values, which we will consider as specific actions in strategic setting.
Another difference with our work is that they do not look at the relationship between causality in the original causal model and strategies in the CGS.

A more similar approach to ours was defined by Baier et al. \cite{baier2021game-theoretic}, but they use extensive form games rather than CGS, and  do not distinguish between agent and environment variables. 
Furthermore, while they do show a result relating actual causality in the causal model to some type of strategy in their extensive form game, they only do this for but-for causes, where we consider the modified HP definition as well. 

Hammond et al. translate the causal model to a multi-agent influence diagram (MAID) that includes utility variables, with the primary goal of studying rational outcomes of the grand coalition \cite{hammond2023ReasoningCausalityGames}.
They hence take a game-theoretic approach, where we take a logic-based approach by focusing on strategic abilities of coalitions of agents.
Nevertheless, we could also apply a game-theoretic analysis to our model, by extending our CGS to include utility variables.
This is however beyond the scope of this paper.

\subsection{Defining a Causal CGS}
In this section we will propose a systematic approach to generate a causal concurrent game structure based on a strongly recursive structural causal model.
We will use the notion of first-level, second-level and higher-level variables as explained in the previous section to determine in which order the agents of the causal model will get to take actions.
For this we define the notion of agent rank:
\begin{definition}
    An \emph{agent ranking function} of a causal model $\mathcal{M}$ is a function $\rho: \mathcal{V} \rightarrow \set{0,...,n}$, where $n$ is the number of distinct variable levels for agent variables in $\mathcal{M}$, such that for all $A,B \in V_a$, $\rho(A) > \rho(B) > 0$ if and only if the variable level of $A$ is higher than the variable level of $B$, and $\rho(A) = \rho(B)$ if and only if $A$ and $B$ have the same variable level. 
    For all $X \in V_e$, $\rho(X) = \rho(A) - 1$ if $\exists A \in V_a$ such that the variable level of $X$ is lower or equal to the variable level of $A$, and there is no $B \in V_a$ that has a variable level between $X$ and $A$. If such an $A$ does not exist, i.e. if the variable level of $X$ is higher than the variable level of all $A \in V_a$, then $\rho(X) = n$.
    The \emph{agent rank} of a variable $A \in V_a$ is $\rho(A)$.
\end{definition}

\begin{example}\label{ex:agent rank}
    In the semi-automated vehicle example we say that $HD$, $ODS$ and $DA$ are the agent variables. 
    We have that $n = 2$ as $HD$ and $ODS$ are both second-level variables and $DA$ is a third level variable as Example \ref{ex:causal network} discusses. 
    There are hence $2$ distinct variable levels for the agent variables.
    From this, it follows that $\rho(HD) = \rho(ODS) = 1$ and $\rho(DA) = 2$, as the variable level of $DA$ is higher than that of $HD$ and $ODS$ and their agent rank needs to be higher than $0$ and maximally $2$.
    For the environment variables, we have that $\rho(O) = \rho(Att) = \rho(HD) - 1 = 0$, because there are no first-level agent variables, so we need a second-level agent variable like $HD$.
    Finally, we have that $\rho(Col) = 2$, since the variable level of $Col$ is $4$ which is higher than all agent variable levels, and hence the agent rank of $Col$ will be the maximum of $2$.
\end{example}

We will first define several components of the causal CGS separately before putting them all together.
From now on, we will assume that all causal models are recursive and have variables which can only attain finitely many values. Moreover we assume that a set of agent variables $V_a \subseteq \mathcal{V}$ is given. 

\begin{definition}[States of a causal CGS]\label{def:states causal CGS}
Given a causal setting $(\mathcal{M},\mathbf{u})$, let $n = \max_{Y\in V_a} \rho(Y)$ be the maximum value of the agent ranks for the agents in $V_a$ and let $m_i = \prod_{\substack{Y \in V_a,\\ \rho(Y) \leq i}} |\mathcal{R}(Y)|$  be the number of possible combinations of action values for agents with an agent rank of no more than $i$. 
The set of \emph{states of a causal CGS}, $Q$, generated based on  $(\mathcal{M},\mathbf{u})$, is given by:
    \begin{equation*}
        Q = \set{q_{0,0}} \cup \set{q_{i,j}\ |\ 1 \leq i \leq n \text{ and }  0 \leq j < m_i }.
    \end{equation*}
\end{definition}
We call $q_{0,0}$ the \emph{starting state} of the causal CGS.
Later, we will see that the evaluation in a state $q_{i,j}$ follows from the actions of agents whose agent variables have agent rank $i$ or less.
\begin{example}\label{ex:states causal CGS}
    We will use the causal model for the semi-automated vehicle example to define a causal CGS (see Figure \ref{fig:ex causal network}). 
    See Example \ref{ex:agent rank} for the agent rank of all variables of the causal model.
    We start with the setting $(\mathcal{M},\mathbf{u})$ with $\mathbf{u} = (U_{O} = 1,U_{Att} = 0)$.
    The set of states is then
    $Q = \{q_{0,0},q_{1,0},q_{1,1}, q_{1,2},q_{1,3},q_{2,0},q_{2,1},q_{2,2},q_{2,3},q_{2,4},$\\$q_{2,5},q_{2,6},q_{2,7}\}.$
    Note that: \newline
    $\prod_{Y \in V_a, \rho(Y) \leq 1} |\mathcal{R}(Y)| = \prod_{Y \in \set{HD,ODS}} |\mathcal{R}(Y)| = |\set{0,1}\times \set{0,1}| = 4$ and
    $\prod_{Y \in V_a, \rho(Y) \leq 2} |\mathcal{R}(Y)| = \prod_{Y \in \set{HD, ODS, DA}} |\mathcal{R}(Y)|  = 8$, so for $i = 1$, we have $j \in \{0,\ldots,3\}$ and for $i = 2$, we have $j \in \{0,\ldots,7\}$.
    These are all the states, because the maximum value of the agent rank $\rho$ is $2$.
\end{example}

We will now define the agent actions in those states.

\begin{definition}[Actions in a causal CGS]\label{def:actions causal CGS}
    Given a causal setting $(\mathcal{M},\mathbf{u})$ and $Q$ the corresponding set of states as defined by Definition \ref{def:states causal CGS}. The possible \emph{actions} for an agent $k \in \set{1,...,N}$ in a state $q_{i,j} \in Q$ are $d_k(q_{i,j}) = \mathcal{R}(A_k)$, where $A_k$ is the agent variable controlled by agent $k$, and $\rho(A_k) = i+1$. Otherwise $d_k(q_{i,j}) = \set{0}$.
\end{definition}

The intuition behind this definition is that agent variables that are earlier on a causal path will earlier get to take an action as the agent variables later on a causal path depend on them. The order of agent variables on a causal path can be seen as representing a protocol that determines when each agent has to take its action.   
We write $a_k$ to denote an action of agent $k \in N$ and $\mathbf{a}_{i,j} = \langle a_1,...,a_N\rangle$ to denote an action profile taken in a certain state $q_{i,j}$, i.e., all actions taken by all agents in state $q_{i,j}$. It is important to note that for a given index $i$ all states $q_{i,j}$ have the same action profiles that can be taken in them, regardless of the value of $j$. We denote this set with $\mathbf{A}_i$. 
Instead of $d_k$ for agent $k$, we will sometimes write $d_{A_k}$ for the agent variable $A_k$ corresponding to agent $k$.

\begin{example}\label{ex:actions causal CGS}
    We continue with the situation as in Example \ref{ex:states causal CGS}.
    The available actions for each agent in each state are: \begin{equation*}
        \begin{array}{ll}
            d_{{HD}}(q_{0,0}) = d_{{ODS}}(q_{0,0}) = \set{0,1}, & d_{{DA}}(q_{0,0}) = \emptyset, \\
            d_{{HD}}(q_{1,j}) = d_{{ODS}}(q_{1,j}) = \emptyset,  & d_{{DA}}(q_{1,j}) = \set{0,1}, \\
            \forall j \in \set{0,\ldots,3},  \text{ and} & \\
            d_{{HD}}(q_{2,j}) = d_{{ODS}}(q_{2,j}) = \emptyset,  & d_{{DA}}(q_{2,j}) = \emptyset, \\
            \forall j \in \set{0,\ldots,7}.&
        \end{array}
    \end{equation*}
\end{example}

These actions must of course lead to transitions to new states. 
\begin{definition}[Transitions in a causal CGS]\label{def:transitions causal CGS}
    Given a causal setting $(\mathcal{M},\mathbf{u})$, $Q$ the corresponding set of states as defined by Definition \ref{def:states causal CGS} and actions as defined by Definition \ref{def:actions causal CGS}, the state following from the action profile $\mathbf{a}_{i,j} \in \mathbf{A}_i$, with $i < \max_{X\in V_a} \rho(X)$, is given by the transition function $\delta$, where $\delta(q_{i,j},\mathbf{a}_{i,j}) = q_{i+1,j'}$ and $|\mathbf{A}_i| \cdot j \leq j'\leq |\mathbf{A}_i| \cdot (j+1) -1$, under the condition that if $\mathbf{a}_{i,j} \neq \mathbf{a}'_{i,j}$, then $\delta(q_{i,j},\mathbf{a}_{i,j}) \neq \delta(q_{i,j}, \mathbf{a}'_{i,j})$.
    If $i = \max_{X\in V_a} \rho(X)$, we define $\delta(q_{i,j},\mathbf{a}_{i,j}) = q_{i,j}$. In this case, there is only one possible action profile $\mathbf{a}_{i,j}$ consisting of only the $0$ action. 
\end{definition}
This definition simply says that every unique action profile in a state leads to a unique new state. This leads to the causal CGS having a tree structure. It is impossible to return to an earlier state and every node can only branch out
\begin{example}\label{ex:transitions causal CGS}
    Continuing with our running example, we will write $\langle 1, 0, 0 \rangle$ for the action profile $\langle HD = 1, ODS = 0, DA = 0 \rangle$. 
    We get that the transitions are:
    \newline 
    \begin{equation*}
        \begin{array}{ll}
            \delta(q_{0,0}, \langle 0, 0, 0\rangle) = q_{1,0},  & 
            \delta(q_{0,0}, \langle 0, 1, 0\rangle) = q_{1,1}, \\
            \delta(q_{0,0}, \langle 1, 0, 0\rangle) = q_{1,2},  & 
            \delta(q_{0,0}, \langle 1, 1, 0\rangle) = q_{1,3}, \\
            \delta(q_{1,0},\langle 0, 0, 0\rangle) = q_{2,0}, &
            \delta(q_{1,0}, \langle 0, 0, 1\rangle) = q_{2,1}, \\
            \delta(q_{1,1},\langle 0, 0, 0\rangle) = q_{2,2}, &
            \delta(q_{1,1}, \langle 0, 0, 1\rangle) = q_{2,3}, \\
            \delta(q_{1,2},\langle 0, 0, 0\rangle) = q_{2,4}, &
            \delta(q_{1,2}, \langle 0, 0, 1\rangle) = q_{2,5}, \\
            \delta(q_{1,3},\langle 0, 0, 0\rangle) = q_{2,6}, &
            \delta(q_{1,3}, \langle 0, 0, 1\rangle) = q_{2,7}, \\
            \delta(q_{2,j},\langle 0, 0, 0\rangle) = q_{2,j} & \forall j \in \set{0,\ldots,7}.
        \end{array}
    \end{equation*}
\end{example}

Now that we have states, actions and transitions, we just need the evaluations of the states.
The evaluation of a state will depend on an initial causal setting and the actions the agents have taken up to this state. The agents fully determine the values of the agent variables, the environment variables follow from these values and the context that was used to define the causal CGS. 

\begin{definition}[Evaluation of states in a causal CGS]\label{def:evaluations in a causal CGS}
    Given a causal setting, $(\mathcal{M},\mathbf{u})$, the set of all possible propositions for the generated causal CGS is $\Pi = \{ X= x \ | \ X \in \mathcal{V}, x \in \mathcal{R}(X)\}$.
    The valuation of each state $q_{i,j} \in Q$, with $Q$ the set of states of the causal CGS according to Definition \ref{def:states causal CGS}, is defined recursively by the \emph{labelling function} $\pi$, as:
    \begin{equation*}
        \begin{array}{ll}
            \pi(q_{0,0})  &= \set{Y = y \ | \ \csetting \models Y = y} \\
            \pi(\delta(q_{i,j}, \myvec{a}_{i,j})) &= \set{Y = y \ | \ \csettingint{\myvec{X}_{i,j}\leftarrow \myvec{x}_{i,j}, \myvec{A}_{i,j} \leftarrow \myvec{a}_{i,j}}\models Y = y},
        \end{array}
    \end{equation*}
    where $\myvec{a}_{i,j}$ is an action profile for state $q_{i,j}$, $\myvec{A}_{i,j} \leftarrow \myvec{a}_{i,j} := \{A_k\leftarrow a_k \ | \ A_k \in V_a, \rho(A_k) = i+1 \text{ and } a_k \in \myvec{a}_{i,j}\}$ is an intervention constructed based on action profile $\myvec{a}_{i,j}$, and $\myvec{X}_{i,j}\leftarrow \myvec{x}_{i,j}$ is recursively defined by: $\myvec{X}_{i+1,j'} \leftarrow \myvec{x}_{i+1,j'}:= \myvec{X}_{i,j}\leftarrow \myvec{x}_{i,j} \cup \myvec{A}_{i,j}\leftarrow \myvec{a}_{i,j}$, if $\delta(q_{i,j}, \myvec{a}_{i,j}) = q_{i+1,j'}$ with $\myvec{X}_{0,0} \leftarrow \myvec{x}_{0,0}= \emptyset$.
\end{definition}

Definition \ref{def:evaluations in a causal CGS} says that an agent action leads to an intervention on the causal setting the causal CGS was based upon. 
We can see $\mathbf{A}_{i,j} \leftarrow \mathbf{a}_{i,j}$ as the intervention that directly follows from the agent action(s) taken in the state $q_{i,j}$, $\mathbf{X}_{i,j} \leftarrow \mathbf{x}_{i,j}$ stores the previous interventions that were made leading up to the state $q_{i,j}$.
We will illustrate this in the following example.
\begin{example}\label{ex:evaluations CGS}
    We continue with the situation as in Example \ref{ex:transitions causal CGS}. 
    We start with the causal setting where $U_O = 1$ and $U_{Att} = 0$, so $\pi(q_{0,0}) = \set{O, \neg Att, HD, ODS, \neg DA, \neg Col}$.
    To determine $\pi(q_{1,0}) = \pi(\delta(q_{0,0}, \langle 0,0,0 \rangle ))$, we need $\myvec{A}_{0,0} \leftarrow \myvec{a}_{0,0} = \set{HD \leftarrow 0, ODS \leftarrow 0}$. This gives us that $\pi(q_{1,0}) = \set{Y = y \ | \ \csettingint{HD \leftarrow 0, ODS \leftarrow 0} \models Y = y} $ $= \set{O, \neg Att, \neg HD, \neg ODS, \neg DA, \neg Col}$.
    Similarly we can determine that $\pi(q_{1,1}) = \{O, \neg Att, \neg HD,  ODS, \neg DA, \neg Col\}$, $\pi(q_{1,2}) = \{O, \neg Att, $\\$ HD, \neg ODS, DA, Col\}$ and $\pi(q_{1,3}) = \set{O, \neg Att, HD,  ODS, \neg DA, \neg Col}$.

    Let us now look at $\pi(q_{2,1}) = \pi(\delta(q_{1,0}, \langle 0,0,1\rangle))$.
    We need $\myvec{X}_{1,0} \leftarrow \myvec{x}_{1,0} = (\myvec{X}_{0,0} \leftarrow \myvec{x}_{0,0} \ \cup \ \myvec{A}_{0,0} \leftarrow \myvec{a}_{0,0}) =  \emptyset \ \cup \set{HD \leftarrow 0, ODS \leftarrow 0}$ as we determined above.
    The new $\myvec{A}_{1,0} \leftarrow \myvec{a}_{1,0} = \set{DA \leftarrow 1}$ and so $\pi(q_{2,1}) = \set{Y = y \ | \ \csettingint{HD \leftarrow 0, ODS \leftarrow 0, DA \leftarrow 1} \models Y = y } = \{O, \neg Att, \neg HD, \neg ODS, DA, \neg Col\}$.
    The valuations for the other states are determined similarly (and are shown in Figure \ref{fig:causal cgs vehicle}).

\end{example}

Now that we have these four definitions, we can give the full definition of a causal CGS.

\begin{definition}[Causal CGS]\label{def:causal CGS}
    Given a causal setting, $(\mathcal{M},\mathbf{u})$, a \emph{causal concurrent game structure} is defined as a tuple $GS = \langle N, Q, d,$\\$ \delta, \Pi, \pi\rangle$ where $N = |V_a|$, every agent only controls one agent variable, $Q$ is a set of states, as defined by Definition \ref{def:states causal CGS}. For every agent $k \in \set{1,...,N}$, $d_k(q_{i,j})$ gives the moves available to this agent in state $q_{i,j} \in Q$, as given by Definition \ref{def:actions causal CGS}.
    The transition function $\delta$ is defined as in Definition \ref{def:transitions causal CGS}.
    The set of possible propositions $\Pi$ and the valuation function $\pi$ are given by Definition \ref{def:evaluations in a causal CGS}.
\end{definition}

We can now add the results of the previous examples together and give a full causal CGS for the semi-automated vehicle example.

\begin{example}\label{ex:cgs rock-throwing}
     Using Definition \ref{def:causal CGS}, we define $N = |V_a| = |\set{HD, ODS, DA}| = 3$. This gives us a full causal CGS, illustrated in Figure \ref{fig:causal cgs vehicle}.

     \begin{figure}[ht]
    \centering
    \setlength{\unitlength}{0.9cm}
    \begin{picture}(7,7.3)(-0.25,0.8)

        \put(0,4.5){\circle{0.8}}
        \put(2,2){\circle{0.8}}
        \put(2,4){\circle{0.8}}
        \put(2,5){\circle{0.8}}
        \put(2,7){\circle{0.8}}
        \put(4,1){\circle{0.8}}
        \put(4,2){\circle{0.8}}
        \put(4,3){\circle{0.8}}
        \put(4,4){\circle{0.8}}
        \put(4,5){\circle{0.8}}
        \put(4,6){\circle{0.8}}
        \put(4,7){\circle{0.8}}
        \put(4,8){\circle{0.8}}
        
        \put(-0.23,4.5){\makebox(0,0)[l]{\footnotesize{$q_{0,0}$}}}
        \put(1.77,7){\makebox(0,0)[l]{\footnotesize{$q_{1,0}$}}}
        \put(1.77,5){\makebox(0,0)[l]{\footnotesize{$q_{1,1}$}}}
        \put(1.77,4){\makebox(0,0)[l]{\footnotesize{$q_{1,2}$}}}
        \put(1.77,2){\makebox(0,0)[l]{\footnotesize{$q_{1,3}$}}}
        \put(3.77,8){\makebox(0,0)[l]{\footnotesize{$q_{2,0}$}}}
        \put(3.77,7){\makebox(0,0)[l]{\footnotesize{$q_{2,1}$}}}
        \put(3.77,6){\makebox(0,0)[l]{\footnotesize{$q_{2,2}$}}}
        \put(3.77,5){\makebox(0,0)[l]{\footnotesize{$q_{2,3}$}}}
        \put(3.77,4){\makebox(0,0)[l]{\footnotesize{$q_{2,4}$}}}
        \put(3.77,3){\makebox(0,0)[l]{\footnotesize{$q_{2,5}$}}}
        \put(3.77,2){\makebox(0,0)[l]{\footnotesize{$q_{2,6}$}}}
        \put(3.77,1){\makebox(0,0)[l]{\footnotesize{$q_{2,7}$}}}

        \put(0.25,4.8){\vector(3,4){1.45}}
        \put(0.4,4.6){\vector(4,1){1.2}}
        \put(0.4,4.4){\vector(4,-1){1.2}}
        \put(0.25,4.2){\vector(3,-4){1.45}}
        \put(2.35,7.2){\vector(2,1){1.3}}
        \put(2.4,7){\vector(1,0){1.2}}
        \put(2.35,5.2){\vector(2,1){1.3}}
        \put(2.4,5){\vector(1,0){1.2}}
        \put(2.4,4){\vector(1,0){1.2}}
        \put(2.35,3.8){\vector(2,-1){1.3}}
        \put(2.4,2){\vector(1,0){1.2}}
        \put(2.35,1.8){\vector(2,-1){1.3}}

        \put(0.45,5.5){\rotatebox{53}{\footnotesize{$\langle 0, 0, 0\rangle$}}}
        \put(0.4,4.7){\rotatebox{14}{\footnotesize{$\langle 0, 1, 0\rangle$}}}
        \put(0.4,4.13){\rotatebox{-14}{\footnotesize{$\langle 1, 0, 0\rangle$}}}
        \put(0.45,3.4){\rotatebox{-53}{\footnotesize{$\langle 1, 1, 0\rangle$}}}
        \put(2.5,7.4){\rotatebox{26.5}{\footnotesize{$\langle 0, 0, 0\rangle$}}}
        \put(3,7){\makebox(0,0)[b]{\footnotesize{$\langle 0, 0, 1\rangle$}}}
        \put(2.5,5.4){\rotatebox{26.5}{\footnotesize{$\langle 0, 0, 0\rangle$}}}
        \put(3,5){\makebox(0,0)[b]{\footnotesize{$\langle 0, 0, 1\rangle$}}}
        \put(3,3.95){\makebox(0,0)[t]{\footnotesize{$\langle 0, 0, 0\rangle$}}}
        \put(2.5,3.45){\rotatebox{-26.5}{\footnotesize{$\langle 0, 0, 1\rangle$}}}
        \put(3,1.95){\makebox(0,0)[t]{\footnotesize{$\langle 0, 0, 0\rangle$}}}
        \put(2.5,1.45){\rotatebox{-26.5}{\footnotesize{$\langle 0, 0, 1\rangle$}}}

        \put(-0.5,4.5){\makebox(0,0)[r]{\scriptsize{$\set{O,\neg Att}$}}}
        \put(1.7,7.5){\makebox(0,0)[b]{\scriptsize{$\set{\neg HD, \neg ODS}$}}}
        \put(1.73,5.5){\makebox(0,0)[b]{\scriptsize{$\set{\neg HD, ODS}$}}}
        \put(1.73,3.5){\makebox(0,0)[t]{\scriptsize{$\set{ HD, \neg ODS}$}}}
        \put(1.73,1.5){\makebox(0,0)[t]{\scriptsize{$\set{ HD, ODS}$}}}
        
        \put(4.43,8){\makebox(0,0)[l]{\scriptsize{$\set{O,\neg Att, \neg HD, \neg ODS, \neg DA, \neg Col}$}}}
        \put(4.43,7){\makebox(0,0)[l]{\scriptsize{$\set{O,\neg Att, \neg HD, \neg ODS, DA, \neg Col}$}}}
        \put(4.43,6){\makebox(0,0)[l]{\scriptsize{$\set{O,\neg Att, \neg HD, ODS, \neg DA, \neg Col}$}}}
        \put(4.43,5){\makebox(0,0)[l]{\scriptsize{$\set{O,\neg Att, \neg HD, ODS, DA, \neg Col}$}}}
        \put(4.43,4){\makebox(0,0)[l]{\scriptsize{$\set{O,\neg Att, HD, \neg ODS, \neg DA, \neg Col}$}}}
        \put(4.43,3){\makebox(0,0)[l]{\scriptsize{$\set{O,\neg Att,  HD, \neg ODS, DA, Col}$}}}
        \put(4.43,2){\makebox(0,0)[l]{\scriptsize{$\set{O,\neg Att, HD, ODS, \neg DA, \neg Col}$}}}
        \put(4.43,1){\makebox(0,0)[l]{\scriptsize{$\set{O,\neg Att,  HD, ODS, DA, Col}$}}}
        
    \end{picture}
    \caption{The causal CGS of the semi-automated vehicle example. We only show the initial values of the variables of agent rank $0$ in the starting state. In the middle states we only show the variables with agent rank corresponding to that state. We also do not show the transitions to the same state in the leaf-states.}
    \label{fig:causal cgs vehicle}
    \Description{A graph depicting the causal concurrent game structure for our running semi-automated vehicle example. The graph has a tree structure with maximal depth 2, there are 4 nodes of depth 1 and 8 leaf nodes. The root has 4 edges leaving it, the depth 1 nodes each have 2 edges leaving it. The root node is named q with the subscript 0,0, it is labelled with the propositions O and negation of Att. The four edges leaving the root node are labelled respectively 0,0,0; 0,1,0; 1,0,0; and 1,1,0. The four depth 1 nodes are named q with the subscripts 1,0 up until 1,3. They are labelled with the propositions: negation of HD, negation of ODS; negation of HD, ODS; HD, negation of ODS; and HD, ODS, respectively. For each of the depth 1 nodes, the edges leaving it are labelled 0,0,0 and 0,0,1. The leaf nodes are named q with the subscript 2,0 up and until 2,7. They are each respectively labelled with the propositions: O, negation of Att, negation of HD, negation of ODS, negation of DA, negation of Col; O, negation of Att, negation of HD, negation of ODS, DA, negation of Col;O, negation of Att, negation of HD, ODS, negation of DA, negation of Col; O, negation of Att, negation of HD, ODS, DA, negation of Col; O, negation of Att, HD, negation of ODS, negation of DA, negation of Col; O, negation of Att, HD, negation of ODS, DA, Col; O, negation of Att, HD, ODS, negation of DA, negation of Col; O, negation of Att, HD, ODS, DA, Col.}
\end{figure}
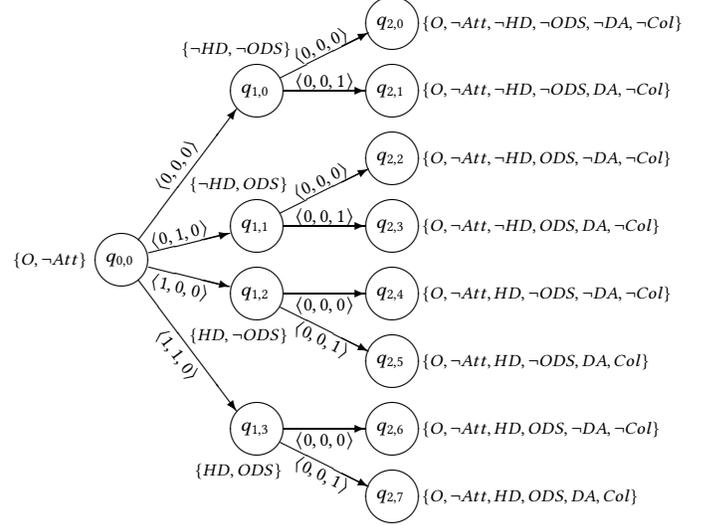
\end{example}

\subsection{Properties of Causal Concurrent Game Structures}
We already mentioned that a causal CGS has a tree structure. In the rest of this paper, we will call states $q_{i,j}$, with $i = \max_{X\in \mathcal{V}} \rho(X)$, the \emph{leaf-states}. 
We will call actions in states where an agent does not control a variable, i.e. $a_k = 0$, when $d_k(q_{i,j}) = \set{0}$, with $\rho(X) \neq i + 1$, \emph{no-op actions}. It is also useful to define an \emph{action path} for a state $q_{i,j}$, that contains all the non no-op actions that led to the state. In other words, the action path contains only the actions that agents took in a state where they could actually choose an action. We will denote this sequence of actions as $\alpha[q_{i,j}]$. 
Formally, for $0 \leq k \leq N$, an action $a_k$ is in this set of actions  $\alpha[q_{i,j}]$ if and only if $\rho(A_k)\leq i$ and there exists an action profile $\mathbf{a}_{i',j'}$, containing $a_k$, such that $q_{i',j'} \in \lambda[q_{i,j} , i]$  (the history of $q_{i,j}$) and $\delta(q_{i',j'} , \mathbf{a}_{i',j'}) \in \lambda[q_{i,j} , i]$. In other words, an action is on the action path for a state $q_{i,j}$, if the state $q_{i',j'}$ in which the action is taken lies on the history of $q_{i,j}$, and the successor of $q_{i',j'}$ can be reached when taking this action.

Our first result is on the size of the causal CGS.
\begin{proposition}
Let $\mathcal{M} = (\mathcal{S},\mathcal{F})$ be a causal model. The size of the causal CGS generated by $\mathcal{M}$ is linear in the size of the extension of $\mathcal{F}$.
\end{proposition}
\begin{proof}
Consider a structural causal model $\mathcal{M} = (\mathcal{S},\mathcal{F})$. Observe that $\mathcal{F}$ specifies the value of each variable for all possible combinations of values of all other variables. Hence $\mathcal{F}$ corresponds to a table of size $|\mathcal{V}| \times \prod_{X \in \mathcal{V}} |\mathcal{R}(X)|$ (the number of cells), which is actually the extension of $\mathcal{F}$. We now show that the number of states in the causal CGS is $O( \prod_{Y \in V_a} |\mathcal{R}(Y)|)$.

By Definition \ref{def:states causal CGS} we have that the number of states of the causal CGS, is given by $|Q| = 1+ \sum_{i = 1}^{n} \prod_{\substack{Y \in V_a,\\ \rho(Y) \leq i}} |\mathcal{R}(Y)|$, where $n = \max_{Y\in V_a} \rho(Y)$. The number of leaf-states is hence given by $\prod_{\substack{Y \in V_a}} |\mathcal{R}(Y)| =: |R(V_a)|$.
The number of states for $i = n-1$ will be at most half $|R(V_a)|$, as there will be at least one variable of rank $n$ that is hence not included in $\prod_{\substack{Y \in V_a,\\ \rho(Y) \leq n-1}} |\mathcal{R}(Y)|$, and this variable will have at least two possible values. We can continue this argument until $i = 1$, which shows us that $|Q|$ is bounded by $1 + \frac{1}{2^{n-1}} |R(V_a)|+\dots +\frac{1}{2} |R(V_a)|+|R(V_a)| \leq 2 |R(V_a)|$. 
Hence the number of states in the causal CGS is $O( \prod_{Y \in V_a} |\mathcal{R}(Y)|)$.
Since a causal CGS is a tree and each state has at most one predecessor, the number of transitions (the size of $\delta$) is also 
$O( \prod_{Y \in V_a} |\mathcal{R}(Y)|)$, hence linear in the size of $\mathcal{F}$ in the original model.
\end{proof}

The statement in the following lemma is a direct consequence of the way the valuation of states is determined in a causal CGS. It states that a variable value cannot change in states corresponding to a higher agent rank than the agent rank of the variable itself.

\begin{lemma}\label{lem:no change after i}
    Let $GS$ be a causal CGS generated by the causal model $\mathcal{M}$. For any endogenous causal variable $X \in \mathcal{V}$ of $\mathcal{M}$, with $\rho(X) = i$, it holds that $(X = x) \in \pi(q_{i,j})$ for some state $q_{i,j}$ of $GS$, if and only if $(X = x) \in \pi(q_{i',j'})$ for all states $q_{i',j'}$ that are descendants of $q_{i,j}$.
\end{lemma}
\begin{proof}
    Let $(X = x) \in \pi(q_{i,j})$. 
    Variable values can change in a state due to interventions, but the only new interventions done in states descended from $q_{i,j}$ are interventions on variables with an agent rank higher than $i$. 
    $X$ has agent rank $i$, so by the definition of agent rank none of those variables can be ancestors of $X$. They are hence unable to influence the value of $X$. 
    Therefore $(X = x) \in \pi(q_{i',j'})$ for all states $q_{i',j'}$ descended from $q_{i,j}$.

    Now, let $(X=x)\in \pi(q_{i',j'})$ for all states $q_{i',j'}$ that are descended from $q_{i,j}$.
    The value of $X$ was not changed in any of those states, because the value of $X$ can only change due to an intervention on $X$ or an ancestor variable of $X$, so only due to variables of agent rank smaller or equal to $\rho(X)$.
    The only interventions on variables that happen in the descendants of $q_{i,j}$ are on variables of agent rank higher than $\rho(X)$, hence $X$ must have had the same value in $q_{i,j}$, i.e. $(X = x) \in \pi(q_{i,j})$.
\end{proof}

We define the notion of \emph{correspondence} to talk about how states in a causal CGS connect to a causal model.
\begin{definition}[Correspondence]\label{def:correspondence}
    We say that a state $q_{i,j}$ of a causal CGS \emph{corresponds} to a causal setting 
    $(\mathcal{M}^{\myvec{Y} \leftarrow \myvec{y}},\mathbf{u})$, where $\myvec{Y} \subseteq \mathcal{V}$, if for all causal variables $X$ of $\mathcal{M}$,
    $(X = x) \in \pi(q_{i,j})$ if and only if $(\mathcal{M}^{\myvec{Y} \leftarrow \myvec{y}},\mathbf{u}) \models X = x$.\footnote{So the causal variable $X$ has value $x$ in the causal setting $(\mathcal{M}^{\myvec{Y} \leftarrow \myvec{y}},\mathbf{u})$.}
\end{definition}
We will sometimes say that a causal setting $(\mathcal{M}^{\myvec{Y} \leftarrow \myvec{y}},\mathbf{u})$ corresponds to a state $q_{i,j}$ of a causal CGS and mean the same thing.
Note that the set $\myvec{Y}$ could also be empty. Hence the causal model $\mathcal{M}^{\myvec{Y} \leftarrow \myvec{y}}$ in Definition \ref{def:correspondence} could also be $\mathcal{M}$.

We can show that a leaf-state of a causal CGS corresponds to a causal setting $\csettingint{\myvec{Y} \leftarrow \myvec{y}}$, where $\myvec{Y} \leftarrow \myvec{y}$ depends on the action path that leads to the leaf-state. 
This connects the definition of causal CGS to the theory of causal models.
\begin{proposition}\label{prop:state correspondence}
    Let $GS$ be a causal CGS generated by a causal setting $\csetting$. If $q_{n,m}$ is a leaf-state of $GS$, then $q_{n,m}$ corresponds to the causal setting $(\mathcal{M}^{\myvec{Y}\leftarrow \myvec{y}}, \mathbf{u})$, where $\myvec{Y} \leftarrow \myvec{y} = \{A_k \leftarrow a_k \ | \ A_k \in V_a \text{ and } a_k \in \alpha[q_{n,m}]\}$, with $\alpha[q_{n,m}]$ the action path for $q_{n,m}$.
\end{proposition}
\begin{proof}
    By Definition \ref{def:evaluations in a causal CGS}, $(X = x) \in \pi(q_{n,m})$ if and only if $\csettingint{\myvec{X}_{i,j}\leftarrow \myvec{x}_{i,j}, \myvec{A} \leftarrow \myvec{a}}\models X = x$, 
    where $\myvec{A} \leftarrow \myvec{a}$ are the actions taken in the state before $q_{n,m}$, and $\myvec{X}_{i,j}\leftarrow \myvec{x}_{i,j}$ are all previously taken actions. Hence $\myvec{Y} \leftarrow \myvec{y} = (\myvec{A} \leftarrow \myvec{a}) \cup \myvec{X}_{i,j}\leftarrow \myvec{x}_{i,j}$ and the proposition is proven.
\end{proof}

This gives us a solid grasp on how a causal CGS relates to the causal model that generates it.
We will use this in the next section when we talk about the connection between agent strategies in a causal CGS and causality in this structural causal model.

\section{Causality in Causal CGS}
\label{sec:causality in CGS}
Now that we have defined causal concurrent game structures and shown what their states represent, it is time to look at how we can use them.
In this section, we will show some relations between causal CGS and the modified HP definition of actual causality, but we first introduce the notion of a causal strategy profile.

From now on, we will denote the set of all agents in a model by $\Sigma$. 
Specifically, for a causal CGS, $\Sigma = \set{k\ | \ X_k \in V_a}$. This set will also be called the \emph{grand coalition} at times.
We will use the notation $F_{X_k = x}$ to denote the strategy for agent $k$ where it takes action $x$ as its non no-op action.
Formally,
\[
    F_{X_k = x}(q_{i,j}) = \left\{ \begin{array}{lll}
        x & \text{if } &  \rho(X_k) = i+1 \\
         0 & \text{else}&
    \end{array}\right.
\]
For a set of agents $\myvec{X}$, we write $F_{\myvec{X} = \myvec{x}}$ to indicate the set of strategies $\{F_{X_k = x} \ | \ $ $ X_k \in \myvec{X}, x \in \myvec{x}\}$.
Let $F_A$ be a strategy for a set of agents $A$, and $F_B$ a strategy for a set of agents $B$. 
Following notation in \cite{Brihaye_DaCosta_Laroussinie_Markey_2008}, we will write $F_A \circ F_B$ to denote a strategy profile for the agents in $A \cup B$ that follows strategy $F_A$ for agents in $A$ and strategy $F_B$ for agents in $B \bs A$.

We define the causal strategy profile as a way to capture the `normal' behaviour of agents when they would follow the causal model.
\begin{definition}[Causal Strategy Profile]\label{def:complete causal strat profile}
Given a causal setting $(\mathcal{M},\mathbf{u})$ and the causal CGS generated by this setting.
Define the \emph{causal strategy profile} $F_\mathcal{M}$ as $F_\mathcal{M} = \set{F_{X_k}\ |\ k \in \Sigma}$, where $F_{X_k}(q_{i,j}) = 0$ if $\rho(X_k) \neq i+1$, and $F_{X_k}(q_{i,j}) = x_k$ otherwise, where $x_k$ is such that $(\mathcal{M},\mathbf{u}) \models [\myvec{X}\leftarrow\myvec{x}] X_k = x_k$, with $\myvec{X} = \set{X_{k'} \ | \ \rho(X_{k'})<\rho(X_k)}$ and $\myvec{x} = \set{x_{k'} \ | \ x_{k'} \in \alpha[q_{i,j}]}$.
\end{definition}
Recall that $\alpha[q_{i,j}]$ is the action path up to state $q_{i,j}$.
If we want an agent $k$ to follow a strategy $F_k$ and the rest of the agents to follow the causal strategy profile, we denote this as $F_k \circ F_\mathcal{M}$.
If a set of agents follows the causal strategy profile, that means that in every state, the agents take the actions that assign those values to the agent variables that they would also have gotten in the causal setting on which the causal CGS is based, given the actions of the other agents.

\begin{example}
    In the semi-automated vehicle example, given the setting where $U_O = 1$ and $U_{Att} = 0$, the causal strategy profile $F_\mathcal{M}$ is such that the human driver does not brake, but the obstacle detection system detects the obstacle.
    The driving assistant brakes in this case, but whenever one of the $HD$ or $ODS$ performs another action, $DA$ does not brake. The causal strategy profile for a causal CGS generated by this causal setting is given in Figure \ref{fig:causal cgs vehicle strategy}.
    \begin{figure}[h]
    \centering
    \setlength{\unitlength}{0.9cm}
    \begin{picture}(7,7.3)(-0.25,0.8)

        \multiput(0.25,4.8)(0.09,0.12){16}{\circle*{0.01}}
        \put(1.7,6.7){\vector(3,4){0}}
        \multiput(0.4,4.6)(0.12,0.03){10}{\circle*{0.01}}
        \put(1.6,4.9){\vector(4,1){0}}
        \multiput(0.4,4.4)(0.12,-0.03){10}{\circle*{0.01}}
        \put(1.6,4.1){\vector(4,-1){0}}
        \put(0.25,4.2){\vector(3,-4){1.45}}
        \put(2.35,7.2){\vector(2,1){1.3}}
        \multiput(2.4,7)(0.12,0){10}{\circle*{0.01}}
        \put(3.6,7){\vector(1,0){0}}
        \put(2.35,5.2){\vector(2,1){1.3}}
        \multiput(2.4,5)(0.12,0){10}{\circle*{0.01}}
        \put(3.6,5){\vector(1,0){0}}
        \multiput(2.4,4)(0.12,0){10}{\circle*{0.01}}
        \put(3.6,4){\vector(1,0){0}}
        \put(2.35,3.8){\vector(2,-1){1.3}}
        \put(2.4,2){\vector(1,0){1.2}}
        \multiput(2.35,1.8)(0.1,-0.05){13}{\circle*{0.01}}
        \put(3.65,1.15){\vector(2,-1){0}}

        \put(0,4.5){\circle{0.8}}
        \put(2,2){\circle{0.8}}
        \put(2,4){\circle{0.8}}
        \put(2,5){\circle{0.8}}
        \put(2,7){\circle{0.8}}
        \put(4,1){\circle{0.8}}
        \put(4,2){\circle{0.8}}
        \put(4,3){\circle{0.8}}
        \put(4,4){\circle{0.8}}
        \put(4,5){\circle{0.8}}
        \put(4,6){\circle{0.8}}
        \put(4,7){\circle{0.8}}
        \put(4,8){\circle{0.8}}
        
        \put(-0.23,4.5){\makebox(0,0)[l]{\footnotesize{$q_{0,0}$}}}
        \put(1.77,7){\makebox(0,0)[l]{\footnotesize{$q_{1,0}$}}}
        \put(1.77,5){\makebox(0,0)[l]{\footnotesize{$q_{1,1}$}}}
        \put(1.77,4){\makebox(0,0)[l]{\footnotesize{$q_{1,2}$}}}
        \put(1.77,2){\makebox(0,0)[l]{\footnotesize{$q_{1,3}$}}}
        \put(3.77,8){\makebox(0,0)[l]{\footnotesize{$q_{2,0}$}}}
        \put(3.77,7){\makebox(0,0)[l]{\footnotesize{$q_{2,1}$}}}
        \put(3.77,6){\makebox(0,0)[l]{\footnotesize{$q_{2,2}$}}}
        \put(3.77,5){\makebox(0,0)[l]{\footnotesize{$q_{2,3}$}}}
        \put(3.77,4){\makebox(0,0)[l]{\footnotesize{$q_{2,4}$}}}
        \put(3.77,3){\makebox(0,0)[l]{\footnotesize{$q_{2,5}$}}}
        \put(3.77,2){\makebox(0,0)[l]{\footnotesize{$q_{2,6}$}}}
        \put(3.77,1){\makebox(0,0)[l]{\footnotesize{$q_{2,7}$}}}

        \put(0.45,5.5){\rotatebox{53}{\footnotesize{$\langle 0, 0, 0\rangle$}}}
        \put(0.4,4.7){\rotatebox{14}{\footnotesize{$\langle 0, 1, 0\rangle$}}}
        \put(0.4,4.13){\rotatebox{-14}{\footnotesize{$\langle 1, 0, 0\rangle$}}}
        \put(0.45,3.4){\rotatebox{-53}{\footnotesize{$\langle 1, 1, 0\rangle$}}}
        \put(2.5,7.4){\rotatebox{26.5}{\footnotesize{$\langle 0, 0, 0\rangle$}}}
        \put(3,7){\makebox(0,0)[b]{\footnotesize{$\langle 0, 0, 1\rangle$}}}
        \put(2.5,5.4){\rotatebox{26.5}{\footnotesize{$\langle 0, 0, 0\rangle$}}}
        \put(3,5){\makebox(0,0)[b]{\footnotesize{$\langle 0, 0, 1\rangle$}}}
        \put(3,3.95){\makebox(0,0)[t]{\footnotesize{$\langle 0, 0, 0\rangle$}}}
        \put(2.5,3.45){\rotatebox{-26.5}{\footnotesize{$\langle 0, 0, 1\rangle$}}}
        \put(3,1.95){\makebox(0,0)[t]{\footnotesize{$\langle 0, 0, 0\rangle$}}}
        \put(2.5,1.45){\rotatebox{-26.5}{\footnotesize{$\langle 0, 0, 1\rangle$}}}

        \put(-0.5,4.5){\makebox(0,0)[r]{\scriptsize{$\set{O,\neg Att}$}}}
        \put(1.7,7.5){\makebox(0,0)[b]{\scriptsize{$\set{\neg HD, \neg ODS}$}}}
        \put(1.73,5.5){\makebox(0,0)[b]{\scriptsize{$\set{\neg HD, ODS}$}}}
        \put(1.73,3.5){\makebox(0,0)[t]{\scriptsize{$\set{ HD, \neg ODS}$}}}
        \put(1.73,1.5){\makebox(0,0)[t]{\scriptsize{$\set{ HD, ODS}$}}}
        
        \put(4.43,8){\makebox(0,0)[l]{\scriptsize{$\set{O,\neg Att, \neg HD, \neg ODS, \neg DA, \neg Col}$}}}
        \put(4.43,7){\makebox(0,0)[l]{\scriptsize{$\set{O,\neg Att, \neg HD, \neg ODS, DA, \neg Col}$}}}
        \put(4.43,6){\makebox(0,0)[l]{\scriptsize{$\set{O,\neg Att, \neg HD, ODS, \neg DA, \neg Col}$}}}
        \put(4.43,5){\makebox(0,0)[l]{\scriptsize{$\set{O,\neg Att, \neg HD, ODS, DA, \neg Col}$}}}
        \put(4.43,4){\makebox(0,0)[l]{\scriptsize{$\set{O,\neg Att, HD, \neg ODS, \neg DA, \neg Col}$}}}
        \put(4.43,3){\makebox(0,0)[l]{\scriptsize{$\set{O,\neg Att,  HD, \neg ODS, DA, Col}$}}}
        \put(4.43,2){\makebox(0,0)[l]{\scriptsize{$\set{O,\neg Att, HD, ODS, \neg DA, \neg Col}$}}}
        \put(4.43,1){\makebox(0,0)[l]{\scriptsize{$\set{O,\neg Att,  HD, ODS, DA, Col}$}}}
        
    \end{picture}
    \caption{The causal CGS of the semi-automated vehicle example. The dotted lines indicate actions that are not following the causal strategy profile.}
    \label{fig:causal cgs vehicle strategy}
    \Description{The same graph for the causal concurrent game structure of the semi-automated vehicle example as before, with the difference that this picture has drawn some edges with dotted lines. These edges are the edges from q subscript 0,0 to q subscript 1,0 up until 1,2, and the edges from q subscript 1,0 to q subscript 2,1, from q subscript 1,1 to q subscript 2,3, from q subscript 1,2 to q subscript 2,4, and from q subscript 1,3 to q subscript 2,7.}
\end{figure}
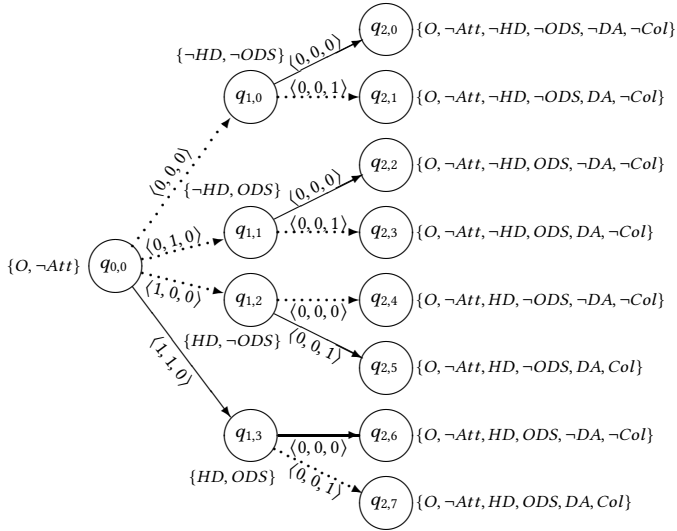
\end{example}

In the following lemma, we relate deviations from the causal strategy profile to interventions in the structural causal model that generated the causal CGS. This can be used to relate agent strategies in the causal CGS to causality in the causal model.
\begin{lemma}\label{lem:causal strat leads to causal setting}
    Let $GS$ be a causal CGS based on a causal setting $\csetting$. If $q_{n,m}$ is the leaf-state of $GS$ that results from the strategy profile $F_{\myvec{X} = \myvec{x}} \circ F_\mathcal{M}$, then $q_{n,m}$ corresponds to $(\mathcal{M}^{\myvec{X} \leftarrow \myvec{x}},\mathbf{u})$.
\end{lemma}
\begin{proof}
    We are going to prove the correspondence, i.e. $(X = x) \in \pi(q_{n,m}) \Leftrightarrow (\mathcal{M}^{\myvec{X} \leftarrow \myvec{x}}, \mathbf{u}) \models X = x$ by induction on the agent rank of $X$.
    
    \noindent \textbf{Base Step:} If $\rho(X) = 0$, $X \in V_e$ and does not depend on any other endogenous variables, so if $(X = x)\in \pi(q_{n,m})$, Lemma \ref{lem:no change after i} 
    implies that $(\mathcal{M}, \mathbf{u}) \models X = x$. Because $X$ does not depend on any agent variables, it will keep the same value when intervening on agent variables $\myvec{X}$, so $(\mathcal{M}^{\myvec{X}\leftarrow\myvec{x}},\mathbf{u}) \models X = x$ as well. For the other way around, if $\csettingint{\myvec{X}\leftarrow\myvec{x}}\models X = x$, by the same reasoning we have that $\csetting \models X = x$ and hence $(X = x) \in \pi(q_{0,0})$, again by Lemma \ref{lem:no change after i} 
    $(X = x) \in \pi(q_{n,m})$ as well.

    \noindent\textbf{Induction Hypothesis:} Suppose that for all $X \in \mathcal{V}$ s.t. $\rho(X) \leq i$, $(X = x) \in \pi(q_{n,m})$ if and only if $\csettingint{\myvec{X} \leftarrow \myvec{x}} \models X = x$.

    \noindent\textbf{Inductive Step:} Let $X$ be such that $\rho(X) = i +1$. First suppose that $X \in V_a$.\newline
     - If $X \in \myvec{X}$, $X$ gets value $x \in \myvec{x}$ in $q_{n,m}$, if and only if $\csettingint{\myvec{X} \leftarrow \myvec{x}}$, because it gets the value directly from the intervention $\myvec{X} \leftarrow \myvec{x}$. So it is true in this case.\newline
    - If $X \notin \myvec{X}$, let $(X = x) \in \pi(q_{n,m})$, the value $x$ was determined by $F_\mathcal{M}$, in particular, $\csetting \models [\myvec{Y} \leftarrow \myvec{y}] X = x$, where $\myvec{Y} = \set{Y \ | \ Y \in V_a, \rho(Y) < \rho(X)}$ and $\myvec{y} = \set{y \ | \ y \in \alpha[q_{i,j}]}$, where $q_{i,j}$ is the state on the path to $q_{n,m}$ where $X$ got to take an action. 
    By the inductive hypothesis we know that $\csettingint{\myvec{X} \leftarrow\myvec{x}} \models \myvec{Y} = \myvec{y}$, and hence $\csettingint{\myvec{X} \leftarrow\myvec{x}} \models X = x$ as well, because all variables $X$ depends on have the same values in $\pi(q_{n,m})$ as in $\csettingint{\myvec{X} \leftarrow \myvec{x}}$.
    On the other hand, if $\csettingint{\myvec{X} \leftarrow \myvec{x}} \models X = x$, we know that $x$ is determined only by variables with a lower agent rank, by the inductive hypothesis all those are in $\pi(q_{n,m})$. The value of $X$ in $\pi(q_{n,m})$ is determined by $F_\mathcal{M}$, so $\csetting \models [\myvec{Y} \leftarrow \myvec{y}] X = x'$. 
    \textcolor{black}{All variable-value pairs $Y,y$ are the variable-value pairs from $\csettingint{\myvec{X} \leftarrow \myvec{x}}$, so $x'$ must be $x$, as all variables of lower agent rank have the same value.}\newline
    Now suppose $X \in V_e$, and let $(X = x) \in \pi(q_{n,m})$. $X$ depends only on variables of a lower level, specifically all agent variables of agent rank less or equal than $i + 1$. By the above and the inductive hypothesis, we know that all those variables have the same value in $\csettingint{\myvec{X} \leftarrow \myvec{x}}$, as in $\pi(q_{n,m})$, hence $X = x$ must also be induced by $\csettingint{\myvec{X} \leftarrow \myvec{x}}$, since there are no other interventions done after the agent variables of rank $i+1$ got their final value.
    Now suppose $\csettingint{\myvec{X} \leftarrow \myvec{x}} \models X = x$, $X$ depends only on variables of lower levels, all of those agent variables have the same value in $\csettingint{\myvec{X} \leftarrow \myvec{x}} \models X = x$ as in $\pi(q_{n,m})$ by the inductive hypothesis. Hence $(X = x) \in \pi(q_{n,m})$ as well, since \textcolor{black}{the environment variables follow the causal model in both $\csettingint{\myvec{X} \leftarrow \myvec{x}}$ as in $\pi(q_{n,m})$.}
\end{proof}

The following corollary follows directly from this lemma, it shows that there is a leaf-state in a causal CGS that corresponds to the original causal setting.

\begin{corollary}\label{col:csetting reachable}
    Let $GS$ be a causal CGS based on a causal setting $\csetting$. If $q_{n,m}$ is the leaf-state resulting from all agents following the causal strategy profile $F_{\mathcal{M}}$, then $q_{n,m}$ corresponds to $(\mathcal{M},\mathbf{u})$.
\end{corollary}
\begin{proof}
    This is a special case of Lemma \ref{lem:causal strat leads to causal setting}, where $\myvec{X} = \emptyset$.
\end{proof}

We can check whether this result holds in our semi-automated vehicle example.
We see in Figure \ref{fig:causal cgs vehicle strategy} that if all agents follow the causal strategy profile, they end up in state $q_{2,6}$ with $\pi(q_{2,6}) = \set{O, \neg Att, HD, ODS, \neg DA, \neg Col}$.
The causal CGS was based on the causal setting where there is an obstacle on the road and the driver is not paying attention, in this case we have $\csetting \models O, \neg Att, HD, ODS, \neg DA, \neg Col$ which does correspond to state $q_{2,6}$, as Corollary \ref{col:csetting reachable} predicted.

With Lemma \ref{lem:causal strat leads to causal setting} we can show that if a set of agents $\myvec{X}$ causes $\vphi$ according to the modified HP definition, with a given witness, then in the causal CGS generated by the causal setting that holds this witness fixed, these agents have a strategy to guarantee $\neg \varphi$ in a leaf-state, provided that all other agents follow the causal strategy profile and vice versa.

\begin{proposition}\label{prop:cause iff strat}
    Let $\Gamma = \set{k \ | \ X_k \in \myvec{X}}$ be a set of agents, $\myvec{x}$ a setting for the variables in $\myvec{X}$, and let $\csetting$ be a causal setting with $\csetting\models \vphi$. 
    $\myvec{X} = \myvec{x}$ is, according to the modified HP definition, a cause of causal formula $\varphi$ in this causal setting $\csetting$, with witness $\myvec{W} = \myvec{w}^*$ if and only if in the causal CGS generated by the causal setting, $\csettingint{\myvec{W} \leftarrow \myvec{w}^*}$, $\Gamma$ has a strategy $F_\Gamma$ such that, $\neg \varphi$ will hold in the leaf-state $q_{n,m}$ resulting from the strategy profile $F_\Gamma \circ F_\mathcal{M}$.
\end{proposition}
\begin{proof}
    We first prove the cause to strategy direction.
    In this case, $\myvec{X} = \myvec{x}$ is a cause of $\varphi$, with witness $\myvec{W} = \myvec{w}^*$ so there exists an alternative value for $\myvec{X}$, $\myvec{x}'$ such that $\csetting \models [\myvec{X} \leftarrow \myvec{x}', \myvec{W} \leftarrow \myvec{w}^*] \neg \varphi$.
    Let $F_\Gamma = \{F_{X_k = x} \ | \ x \in \myvec{x}' \text{ if } X_k \in \myvec{X}\}$. 
    By Lemma \ref{lem:causal strat leads to causal setting}, the leaf-state $q_{n,m}$ corresponds to $\csettingint{\myvec{X} \leftarrow \myvec{x}',\myvec{W}\leftarrow \myvec{w}^*}$, and we have that $\csettingint{\myvec{X} \leftarrow \myvec{x}', \myvec{W}\leftarrow\myvec{w}^*}\models\neg \vphi$ and hence $\neg \vphi$ holds in $q_{n,m}$.
    
    Now for the other direction, let $F_\Gamma$ be the strategy such that $\neg \vphi$ will hold in the leaf-state $q_{n,m}$ that results from the strategy profile $F_\Gamma \circ F_{\mathcal{M}}$ in the causal CGS generated by the causal setting $\csettingint{\mathbf{W}\leftarrow\mathbf{w}^*}$.
    Let $\myvec{x}$ be such that $\csetting \models \myvec{X} = \myvec{x}$ and let $\myvec{x}'$ be such that $\myvec{X} = \myvec{x}'\subseteq \pi(q_{n,m})$. By Lemma \ref{lem:causal strat leads to causal setting}, $q_{n,m}$ must correspond to $\csettingint{\myvec{X}\leftarrow\myvec{x}', \myvec{W} \leftarrow \myvec{w}^*}$. Hence $\csetting \models [\myvec{X} \leftarrow \myvec{x}', \myvec{W} \leftarrow \myvec{w}^*]\neg \vphi$ and by definition we have that $\csetting \models \myvec{X} = \myvec{x} \wedge \vphi$. Moreover, $\myvec{x} \neq \myvec{x}'$, because if they were the same it would not be the case that setting $\myvec{X}$ to $\myvec{x}'$ would give a different result than the original causal setting (the determinism axiom of the causal reasoning axioms \cite{halpern2016actual}). Hence $\myvec{X} = \myvec{x}$ is a cause of $\vphi$ according to the modified HP definition, with witness $\myvec{W} = \myvec{w}^*$.
\end{proof}

This result cannot be used to find causes in a causal CGS, because one would already need to know the witness. 
However, we have another result for the causal setting where the witness was not held fixed, provided the witness consists of only agent variables.
The following proposition states that in that case, the set of agents consisting of both the cause and the witness variables has a strategy to guarantee $\neg \varphi$ in a leaf-state, provided all other agents follow the causal strategy profile and vice versa.

\begin{proposition}\label{prop:cause iff superset strat}
    Let $\Gamma = \set{k \ | \ X_k \in \myvec{X}\cup \myvec{W}, \text{ and } \myvec{X} \cup \myvec{W} \subseteq V_a}$ be a set of agents, $\myvec{x},\myvec{w^*}$ are settings for the variables in $\myvec{X},\myvec{W}$ respectively, and let $\csetting$ be a causal setting with $\csetting\models \vphi$. $\myvec{X} = \myvec{x}$ is, according to the modified HP definition, a cause of causal formula $\varphi$ in this causal setting $\csetting$, with witness $\myvec{W} = \myvec{w}^*$ if and only if in the causal CGS generated by this causal setting, $\Gamma$ has a strategy $F_\Gamma$ such that, $\neg \varphi$ will hold in the leaf-state $q_{n,m}$ resulting from the strategy profile $F_\Gamma \circ F_\mathcal{M}$.
\end{proposition}
\begin{proof}
    We first prove the cause to strategy direction.
    In this case, $\myvec{X} = \myvec{x}$ is a cause of $\varphi$, with witness $\myvec{W} = \myvec{w}^*$ so there exists an alternative value for $\myvec{X}$, $\myvec{x}'$ such that $\csetting \models [\myvec{X} \leftarrow \myvec{x}', \myvec{W} \leftarrow \myvec{w}^*] \neg \varphi$.
    Let $F_\Gamma = \{F_{X_k = x} \ | \ k \in \Gamma, x \in \myvec{x}' \text{ if } X_k \in \myvec{X}, \text{ else } x \in \myvec{w}^*\}$. 
    By Lemma \ref{lem:causal strat leads to causal setting}, the leaf-state $q_{n,m}$ corresponds to $\csettingint{\myvec{X} \leftarrow \myvec{x}',\myvec{W}\leftarrow \myvec{w}^*}$, and we have that $\csettingint{\myvec{X} \leftarrow \myvec{x}', \myvec{W}\leftarrow\myvec{w}^*}\models\neg \vphi$ and hence $\neg \vphi$ holds in $q_{n,m}$.
    
    Now for the other direction. 
    Let $\myvec{x}, \myvec{w}^*$ be such that $\csetting \models \myvec{X} = \myvec{x} \wedge \myvec{W} = \myvec{w}^*$ and let $\myvec{x}'$ be such that $\myvec{X} = \myvec{x}'\subseteq \pi(q_{n,m})$. 
    By Lemma \ref{lem:causal strat leads to causal setting}, $q_{n,m}$ must correspond to $\csettingint{\myvec{X}\leftarrow\myvec{x}', \myvec{W} \leftarrow \myvec{w}^*}$. 
    Hence $\csetting \models [\myvec{X} \leftarrow \myvec{x}', \myvec{W} \leftarrow \myvec{w}^*]\neg \vphi$ and by definition we have that $\csetting \models \myvec{X} = \myvec{x} \wedge \vphi$. 
    Moreover, $\myvec{x} \neq \myvec{x}'$, because if they were the same it would not be the case that setting $\myvec{X}$ to $\myvec{x}'$ would give a different result than the original causal setting (the determinism axiom of the causal reasoning axioms \cite{halpern2016actual}). 
    Hence $\myvec{X} = \myvec{x}$ is a cause of $\vphi$ according to the modified HP definition, with witness $\myvec{W} = \myvec{w}^*$.
\end{proof}

As but-for causes have no witness, they give a stronger result.

\begin{corollary}\label{col:but-for cause iff strat}
    Let $\Gamma = \set{k \ | \ X_k \in \myvec{X}}$ be a set of agents, $\myvec{x}$ a setting for the variables in $\myvec{X}$, and let $\csetting$ be a causal setting with $\csetting\models \vphi$. 
    $\myvec{X} = \myvec{x}$ is a but-for cause of causal formula $\varphi$ in this causal setting $\csetting$ if and only if in the causal CGS generated by the causal setting, $\csetting$, $\Gamma$ has a strategy $F_\Gamma$ such that, $\neg \varphi$ will hold in the leaf-state $q_{n,m}$ resulting from the strategy profile $F_\Gamma \circ F_\mathcal{M}$
\end{corollary}
\begin{proof}
    A but-for cause is a special case of the modified HP definition where $\myvec{W} = \emptyset$. This statement is hence a special case of propositions \ref{prop:cause iff strat} and \ref{prop:cause iff superset strat}.
\end{proof}

\begin{example}
    In our running semi-automated vehicle example, both $ODS$ and $\neg DA$ are but-for causes of $\neg Col$, there being no collision (in the causal setting that there is an obstacle and the human driver is not paying attention). 
    In the case of $ODS$ we can define $F_{ODS}$ to be the strategy where the obstacle detection system will not pass on a signal to the driving assistant. 
    If all other agents follow the causal strategy profile, they will reach state $q_{2,5}$.
    Indeed $Col \in \pi(q_{2,5})$.
    Similarly, in the case of $\neg DA$, we can define $F_{DA}$ to be the strategy where the driving assistant does not brake. 
    When the other agents follow $F_{\mathcal{M}}$, they will end up in $q_{2,7}$.
    In that state it is indeed true that $Col \in \pi(q_{2,7})$.
\end{example}

In this section we have shown how agent strategies in a causal CGS relate to the causal relations in the causal setting the causal CGS was based on.
In order to do this, we have introduced the notion of a causal strategy profile, a strategy for the grand coalition that makes sure the agents do exactly those actions they would do if all relations in the causal model would be followed.

\section{Conclusion and Discussion}
\label{sec:conclusion and discussion}
This paper investigates the relation between two formalisms that can be used to model multi-agent systems: structural causal models as introduced by Pearl \cite{pearl1995causal} and concurrent game structures. This is done by proposing a systematic way to translate structural causal models to the co-called causal CGS.  
In such a causal CGS, agents will get to take their actions at a point corresponding to their position in the structural causal model. 
The causal CGS is defined in such a way that the leaf-states correspond to interventions on the original structural causal model.

In this paper we have used the variable levels as defined by Halpern to determine the position of the variables in the causal model \cite{halpern2016actual}. 
However, we can use any function that maps the endogenous variables to the positive integers as long as the function assigns a lower rank to a variable than to its descendants. In general, there are multiple of these functions possible for a given structural causal model. 
The formal results of this paper will hold for all such functions, though the structure of the resulting causal CGS may change due to the specific function used.

We can also relax the assumption that each agent controls exactly one agent variable. 
We assumed this to simplify the presentation of the causal CGS, but it is not a strict requirement. 
In principle an agent could control several variables and perform multiple actions, at several time steps, in the causal CGS.

A limitation of our approach is that in general, we are only able to give a result for actual causes if we already know the witness. For but-for causes, we are able to use the agents' abilities in the causal CGS to determine the but-for causes, 
but in general, this is not possible.
Another limitation is that the causal CGS is generated with respect to a specific causal setting, hence the results only apply to a single context. This means that if the context is uncertain, multiple causal CGS have to be made to evaluate all possible outcomes. However, it is possible that this problem can be solved by using a version of an epistemic CGS. This can be researched in the future.

So far, we have only looked at deterministic and recursive causal models to define the causal CGS.
However, causal relations are often probabilistic and cyclic in many practical use cases. Modelling such cases requires probabilistic and non-recursive causal models to, for example, capture the mutual dependencies between agents.
In order to deal with probabilities, we will have to either employ probabilistic CGS, or use another type of model (e.g. Markov games).
Moreover, allowing cyclic dependencies would make the evaluation of the states difficult, as the variable values would depend on each other.
We think that this could possibly be dealt with by adding a temporal component to the model, but this needs more research.

Another direction of future work would be to use this framework to compare different approaches to defining responsibility in multi-agent settings. Some existing works define responsibility based on causal relations between agents and an outcome (like \cite{Alechina_Halpern_Logan_2020,chockler2004responsibility,Friedenberg_Halpern_2019} and \cite{Beckers2023moral}), while other work is based on whether agents had a strategy to avoid the outcome (like \cite{baier2021game} and \cite{yazdanpanah2019strategic}). The definition of causal CGS might help to combine both directions of research.
Moreover, we can also look at how our approach compares to rule-based approaches to causality. Since Lorini's \cite{lorni2023rule} work shows a correspondence between his rule-based framework for causal reasoning and the structural equations framework, it seems possible that his framework can also be shown to have a connection to our causal CGS.

This research could be used in multi-agent systems with a clear causal structure. Examples of this are traffic control environments, like planes that cannot land when another is departing, trains that cannot travel over the same track at the same time, or traffic lights on a junction that cannot all turn to green at the same time. 
Other applications could be in the analysis of multi-player games, after all, players could cause other players to make a certain move, or even energy management systems, where supply and demand of electricity influence each other. 
In these situations this research could be used to help making decisions, or after something has gone wrong to help attributing responsibility for this.
\balance


\begin{acks}
This publication is part of the CAUSES project (KIVI.2019.004) of the research programme Responsible Use of Artificial Intelligence which is financed by the Dutch Research Council (NWO)
and ProRail.
\end{acks}



\bibliographystyle{ACM-Reference-Format} 
\bibliography{bibliography}


\end{document}